\documentclass[letterpaper, 10 pt, conference]{ieeeconf}  

\IEEEoverridecommandlockouts                              

\usepackage{amssymb,amsmath}

\usepackage{amsthm}
\usepackage{mathtools}
\usepackage{multirow}
\usepackage{graphicx}
\usepackage{comment}
\usepackage[sort,compress,noadjust]{cite}
\usepackage[font=footnotesize]{subcaption}
\usepackage[font=footnotesize]{caption}

\usepackage{amsfonts}
 \usepackage{tabularx}
\usepackage{graphicx}
\usepackage{array}
\usepackage{float}
\usepackage{hhline}
\usepackage{algorithm}
\usepackage{algorithmic}
\usepackage[colorlinks,urlcolor=blue]{hyperref}
\usepackage{setspace}
\usepackage{colortbl}
\usepackage[dvipsnames]{xcolor}
\usepackage[normalem]{ulem}
\usepackage{comment}
\usepackage{arydshln}
\usepackage{accents}

\usepackage{placeins}

\theoremstyle{plain}

\theoremstyle{plain}
\newtheorem{proposition}{Proposition}
\theoremstyle{plain}

\theoremstyle{plain}

\theoremstyle{definition}
\newtheorem{assumption}{Assumption}
\theoremstyle{definition}
\newtheorem{definition}{Definition}
\theoremstyle{remark}
\newtheorem{remark}{Remark}

\usepackage[capitalize]{cleveref}
\crefformat{equation}{(#2#1#3)}
\Crefformat{equation}{Equation~(#2#1#3)}
\Crefname{equation}{Equation}{Eqs.}

\title{\LARGE \bf
 Submodular Optimization for Keyframe Selection \& Usage in SLAM
}

\author{David Thorne$^1$, Nathan Chan$^1$, Yanlong Ma$^1$, Christa S. Robison$^2$, Philip R. Osteen$^2$, Brett T. Lopez$^1$
\thanks{*This research was sponsored by the DEVCOM Army Research Laboratory (ARL) under SARA CRA
W911NF-24-2-0017. Distribution Statement A: Approved for public release; distribution is unlimited.}
\thanks{$^1$ University of California, Los Angeles, Los Angeles, CA, USA {\tt\small \{davidthorne, nchan22, yanlong, btlopez\}@ucla.edu}}%
\thanks{$^{2}$DEVCOM Army Research Laboratory (ARL), Adelphi, MD, USA. \{\texttt{christopher.j.robison5, philip.r.osteen \}.civ@army.mil}}}

\begin{document}

\maketitle
\thispagestyle{empty}
\pagestyle{empty}

\begin{abstract}
Keyframes are LiDAR scans saved for future reference in Simultaneous Localization And Mapping (SLAM), but despite their central importance most algorithms leave choices of which scans to save and how to use them to wasteful heuristics.
This work proposes two novel keyframe selection strategies for localization and map summarization, as well as a novel approach to submap generation which selects keyframes that best constrain localization.
Our results show that online keyframe selection and submap generation reduce the number of saved keyframes and improve per scan computation time without compromising localization performance. 
We also present a map summarization feature for quickly capturing environments under strict map size constraints. 
\end{abstract}

\section{Introduction}
\label{sec:introduction}


LiDAR is an increasingly popular sensing modality for Simultaneous Localization And Mapping (SLAM) because of its superior range measurement accuracy across different environments and conditions.
LiDAR odometry estimation algorithms work by computing a rigid body transform that best aligns the most recent scan, i.e., point cloud, with previously registered point clouds. 
Given the density of LiDAR data, having a SLAM system that saves all incoming point clouds would be infeasible from both a memory and computation point of view.
A more pragmatic approach is to save unique point clouds called keyframes that can be combined to form descriptive submaps for future scan alignment.
Despite keyframes and submaps playing a central role in SLAM, their treatment has largely been left to heuristics which lack generality and require extensive, environment-specific tuning to achieve adequate performance.
Another shortcoming of many LiDAR SLAM pipelines is their inability to create maps on-the-fly tailored for sharing with other agents or downstream processes that require specific information. 
We address the issues of online keyframe selection and usage by leveraging the unique properties of submodular optimization to i) save unique keyframes according to a neural network which encodes point cloud similarity and ii) generate submaps via an optimization of the keyframes which best constrain scan alignment.
We also propose a principled method for generating size-constrained summary maps for sharing or downstream processes.


\begin{figure}[t!]
    \centering
    \includegraphics[width=.48\textwidth]{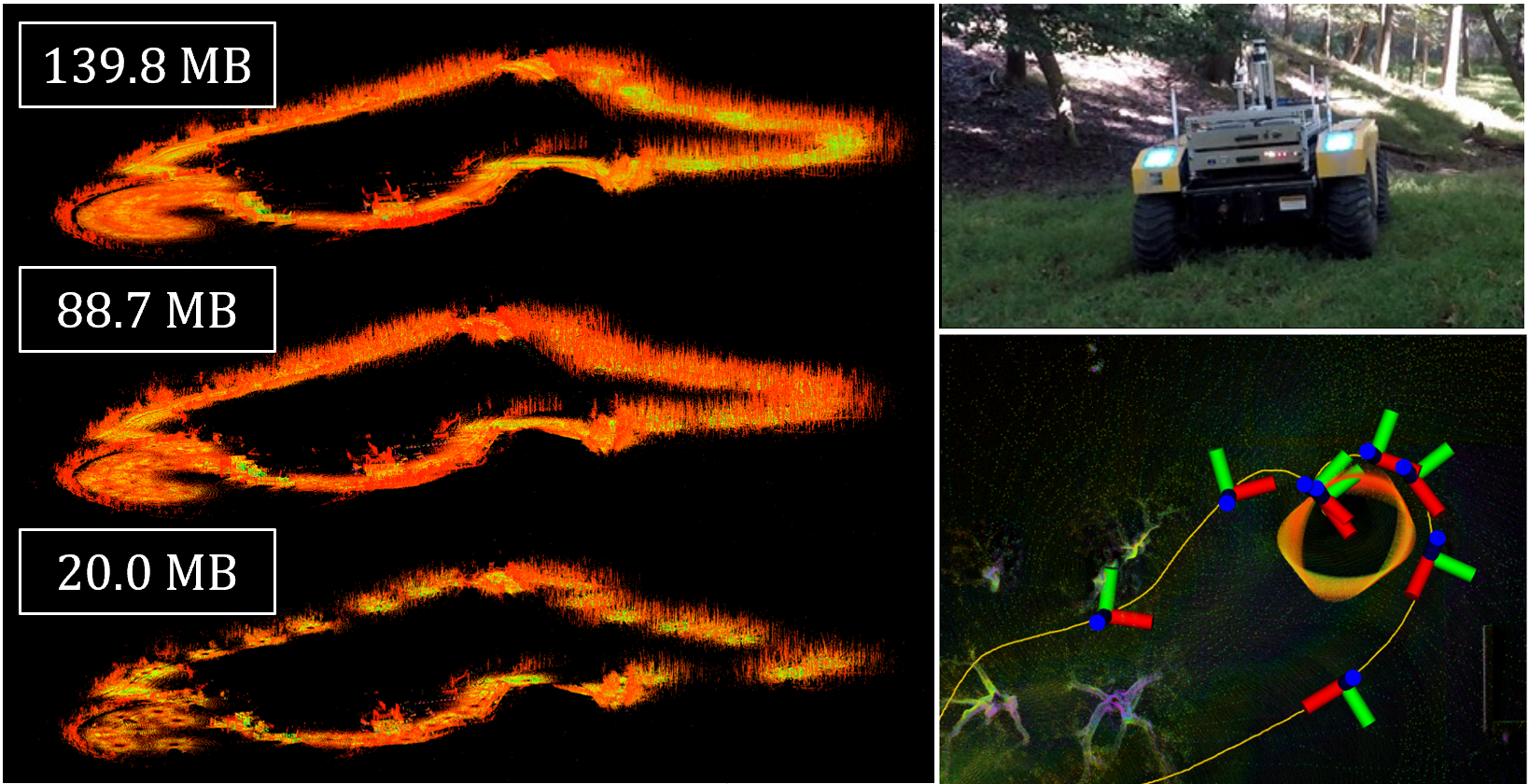}
    \caption{(Left) Dense point cloud map of 2.3km forest loop with smaller summary maps. The dense map was generated during online SLAM and uses 456 keyframes as opposed to the summary maps which use 300 and 75 keyframes and are built in under one second. (Top-right) Modified Clearpath Robotics Warthog robot with LiDAR sensor used for collecting the forest loop dataset. (Bottom-right) Unique keyframe selection, each keyframe (axis) is chosen to capture a unique part of the map.} 
    \label{fig:top-right}
    \vskip -0.3in
\end{figure}


Our key insight is that decisions regarding keyframe selection and usage in addition to map summarization can be formulated as combinatorial optimization problems.
Submodularity, or the mathematically defined notion of diminishing returns, captures the trade-offs inherent in these problems and provides efficient algorithms for approximately solving NP-hard combinatorial problems with suboptimality guarantees. 
Submodular optimization has been used in a variety of robot applications such as single- and multi-agent coverage \cite{zhang2016submodular,corah2019distributed,shi2021communication}, but has been used sparingly in SLAM. 
The most relevant to our work is that of \cite{carlone2018attention} which used submodular optimization for selecting the most useful set of features to track over a window for visual-inertial odometry.
Other SLAM works that utilize submodular optimization have primarily focused on sparsifying pose graphs \cite{khosoussi2019reliable,doherty2022spectral}.


The contributions of this work are threefold. 
First, we propose a keyframe selection strategy which saves point clouds with unique descriptors from a neural network, resulting in a sparser set of keyframes that use less memory than our baseline.
Next, we leverage submodular optimization to generate smaller submaps which better constrain scan alignment, leading to improvements in computation time when combined with sparse keyframe sets.
In our final contribution, we combine the LiDAR neural network and streaming submodular algorithms to generate size-constrained maps for efficient communication and fast operation in downstream processes. 
A key element to our methods is the efficient comparison between point clouds enabled by a neural network which encodes the similarity of LiDAR point clouds as descriptor vectors \cite{ma2022overlaptransformer}.
We demonstrate our work on long duration datasets including a 2.3 km loop (pictured in \cref{fig:top-right}) and a 6.2 km out-and-back at the Army Research Laboratory Graces Quarters test facility.

 
\section{Related Works}
\label{sec:related_works}

The decision to save a point cloud as a keyframe---a process refered to as keyframe selection---must balance localization accuracy with memory usage.
This is because keyframes are used as target point clouds for registration but must be saved indefinitely in memory \cite{kretzschmar2011efficient}.
LiDAR keyframe selection has largely relied on heuristics such as distance from other keyframes \cite{shan2020lio,li2021towards}, which can lead to keyframe sets that are redundant or insufficient for localization depending on the threshold parameter.
A few works have selected sparse keyframe sets via feedback from scan alignment \cite{chen2023dliom, lin2023infola}. 
We use this strategy to supplement a proactive policy which selects keyframes with distinct descriptors. 

Submaps are composed of nearby (local) keyframes and are important for keeping the computation for point cloud registration manageable. 
Despite their importance, few have investigated submap generation and its effect on LiDAR SLAM. 
Instead, most LIO algorithms rely on an excess of nearby keyframes \cite{shan2020lio, li2021towards}.
Past work has demonstrated the value of including keyframes from the edge of the scan \cite{chen2022direct,chen2023direct}; however, larger submaps demand more computational resources.
Aside from keyframe-based methods, other direct LIO algorithms use a sliding window local map with all recent points \cite{xu2022fast,reinke2022locus}.
These methods include all points regardless of their relevance, which inflates the submap size.

One technique that has shown to be effective at (passively) characterizing the quality of submaps/local features is the eigenvalues of the Hessian of the scan alignment optimization \cite{zhang2016degeneracy, ebadi2021dare, tagliabue2021lion, han2023dams}. 
However, these methods only use the Hessian eigenvalues as a passive metric for monitoring localization performance.
A related work used the Hessian eigenvalues to select better correspondences for Iterative Closest Point (ICP) algorithms \cite{gelfand2003geometrically} in generic point cloud alignment.
A similar approach for Visual Inertial Odometry (VIO) \cite{carlone2018attention} showed that optimizing a feature set over the minimum eigenvalue of an information matrix using submodular optimization yielded a near optimal set with low computation times.

In addition to feature selection in VIO, submodular optimization has been used in SLAM to prune LiDAR keyframes \cite{kretzschmar2011efficient}, generate sparse pose graphs \cite{khosoussi2019reliable,doherty2022spectral}, and select navigation anchor points \cite{chen2021anchor}.
Streaming submodular algorithms \cite{badanidiyuru2014streaming, feldman2018less} have been used to summarize large image datasets, but have not been used for LiDAR map summarization or SLAM to the best of our knowledge.

\section{Preliminaries}
\label{sec:problem_formulation}

This work is concerned with formulating and efficiently solving combinatorial optimization problems that arise in SLAM. 
The problem of interest is formally stated as
\begin{equation}
\label{eq:submod_optimization}
    \max_{{S} \subseteq E} \, f({S})  ~~ \text{Subject to constraints on ${S}$,}
\end{equation}
where $f: S \rightarrow \mathbb{R}$ is a set function and $E$ is comprised of all possible discrete elements. 
Often times the constraint on $S$ is simply the number of elements in $S$.
The following definitions and results from submodular optimization will be important for solving \cref{eq:submod_optimization} efficiently.

\begin{definition}
    A set function $f:{S} \rightarrow \mathbb{R}$ is said to be \emph{non-negative monotonic} if $f(\emptyset) \geq 0$ and $f(B) \geq f(A)$ for all subsets $A \subseteq B \subseteq E$.
\end{definition}

\begin{definition}
\label{def:marginal_gain}
    The marginal value of a set function given an added element is defined as $\Delta f(A \cup \{i\}) = f(A \cup \{i\}) - f(A)$ where the element $i \in E$ is added to set $A$.
\end{definition}

\begin{definition}
\label{def:submodularity}
    A set function $f:{S} \rightarrow \mathbb{R}$ is \emph{submodular} if for $A \subseteq B \subseteq E$ and any element $i \in E \, \backslash \, B$ then
\end{definition}
\vskip -0.25in
\begin{align}
    f(A \cup \{i\}) - f(A) \geq f(B \cup \{i\}) - f(B).
\end{align}
Put simply, submodularity is equivalent to diminishing returns.

The celebrated result in submodular optimization is the suboptimality bound of the simple greedy optimization for non-decreasing monotone submodular functions with cardinality constraints $|S|\leq N$, in which the set $S^{\#}$ is built over $N$ iterations of selecting the highest marginal value element.
If $S^{*}$ is the optimal solution to \cref{eq:submod_optimization}, then from \cite{nemhauser1978best}
\begin{align}
    f({S}^{\#}) \geq (1-e^{-1}) f({S}^{*}),
\end{align}
which shows the greedy solution is at worst a factor of 0.63 of the optimal. 
This result justifies using simple greedy optimization to solve $\cref{eq:submod_optimization}$ to a known suboptimality bound.

Functions that are non-negative monotone but are not submodular can still enjoy suboptimality guarantees with greedy selection.
The following two definitions originate in \cite{das2011submodular}, but we use the notation from \cite{carlone2018attention}.

\begin{definition}
\label{def:submod_ratio}
    The submodularity ratio of a non-negative set function $f: S \rightarrow \mathbb{R}$ with respect to a set $A$ is
\end{definition}
\vskip -0.2in
\begin{align}
\label{eq:submod_ratio_eq}
    \gamma_{A} \triangleq \min_{L \subseteq A, S:|S|\leq \kappa, S \cap L = \emptyset} \frac{\sum_{s\in S} (f(L \cup \{s\}) - f(L))} {f(L \cup S) - f(L)},
\end{align}
where $\kappa$ is a strictly positive integer.

\begin{definition}
\label{def:approx_submod_bound}
    Let $f: S \rightarrow \mathbb{R}$ be a non-negative monotone set function and let $S^{*}$ be its optimal solution subject to a cardinality constraint. The set $S^{\#}$ computed by the greedy algorithm is then given by
    \begin{equation}
        f(S^{\#}) \geq (1-e^{-\gamma_{S^{\#}}}) f(S^{*}),
    \end{equation}
    where $\gamma_{S^{\#}}$ is the submodularity ratio of the greedy solution.
\end{definition}

\begin{remark}
    Functions with submodular ratios between 0-1 are henceforth referred to as approximately submodular due to the suboptimality guarantees they enjoy from \cref{def:approx_submod_bound}.
\end{remark}

We conclude with a discussion on streaming submodular algorithms which have proven to be effective at solving combintorial optimization problems with a single pass over the data rather than multiple greedy passes.
The streaming algorithm most pertinent to this work is described in \cref{sub:map_summarization} and given by Algorithm 3, but the intuition is that elements with high marginal value are added to a potential solution based on a guess of the optimal function value.
With a range of optimal value guesses, the streaming algorithm can build several potential solution sets and return the best which is guaranteed to be within a 1/2-$\epsilon$ factor of the optimal set where $\epsilon$ is a granularity constant \cite{badanidiyuru2014streaming}.

\textit{Notation.} We call the set of all collected point clouds $\textbf{E}$. 
We denote the set of keyframes as $\mathcal{K} \subseteq \textbf{E}$, and the local submap as $\mathcal{S} \subseteq \mathcal{K}$.
All sets have a defined ordering function where $\textbf{E}$ and $\mathcal{K}$ are ordered chronologically and the ordering of $\mathcal{S}$ is defined by when an element is added.
We use a matching upper case letter to indicate a member of a set, e.g., $S \in \mathcal{S}$.
The position of an element is given by a subscript index $S_{i}$, and we adopt the notation $\mathcal{S}_{<i}$ to denote the subset of $\mathcal{S}$ such that only elements with index less than $i$ are included, e.g., $\mathcal{S} = \{S_{0},S_{1},S_{2}\}, ~ \mathcal{S}_{<2} = \{S_{0},S_{1}\}$.

\section{LiDAR Scan Neural Network}
\label{sec:LSNN}

We use a discriminative and generalizable function that maps point clouds to global descriptor vectors which describe unique point cloud features. 
More formally, we define a function $\phi: \textbf{E} \rightarrow \mathcal{G} \subset \mathbb{R}^{p}$ such that, for any point clouds $E_a$, $E_i$, $E_j$ $\in \textbf{E}$ and their corresponding descriptors $\phi(E_a)$, $\phi(E_i)$, $\phi(E_j)$ $\in \mathcal{G}$, we have 
$\mathcal{\sigma}(E_i, E_a) \geq  \mathcal{\sigma}(E_j, E_a) \implies ||\phi(E_i) – \phi(E_a)|| \leq ||\phi(E_j) – \phi(E_a)||$
where $\mathcal{\sigma}$ represents similarity between two point clouds.
We normalize descriptors to be on the \emph{p}-dimension hypersphere such that $||\phi(E)||=1$ and $||\phi(E_{i})-\phi(E_{j})|| \leq 2$.

We adopt the neural network proposed by \cite{ma2022overlaptransformer}; the overall framework is presented in Fig. \ref{fig:nn_pipeline}. 
Compared to point- and discretized-based methods \cite{arce2023padloc, komorowski2021minkloc3d, vidanapathirana2022logg3d}, the neural network is compact and generalizable because it can achieve fast execution without any semantic information. 
Instead of using the overlap between the projected range images to supervise the training, we measure the similarity between two point clouds $E_i$ and $E_j$ directly using the 3D Jaccard index \cite{chen2023dliom}.
Defined as $J(E_i, E_j) = {|E_i  \bigcap E_j |} / {|E_i \bigcup E_j|}$, the Jaccard index is the equivalent of overlap in the 3D point cloud domain to provide a direct measure of point cloud similarity which enhances training results.
Unless otherwise specified, the dimension of the output descriptor is $\text{dim}(\mathcal{G}) = {256}$.

\begin{figure}[t!]
    \centering
    \vspace{6pt}
    \includegraphics[width=0.45\textwidth]{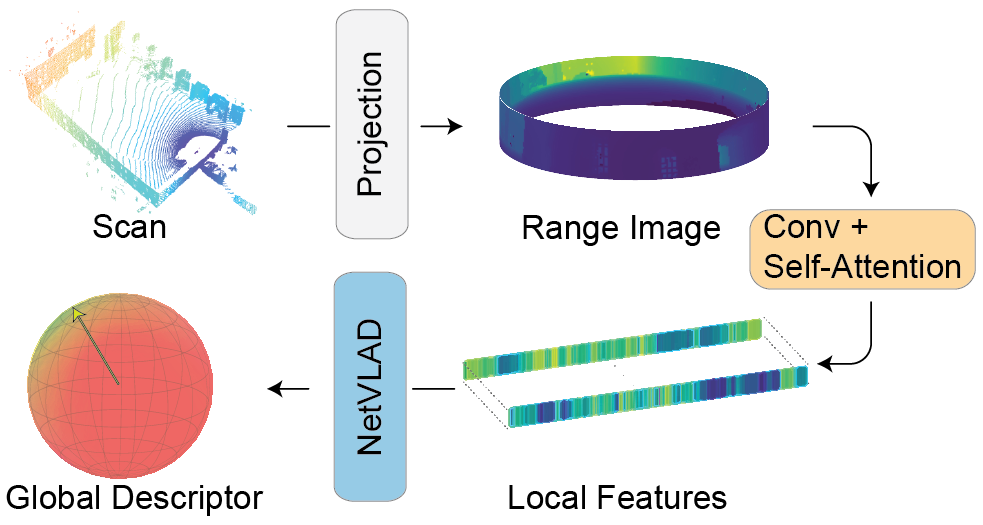}
    \caption{Global descriptor generation. The point cloud is projected into a range image (\textcolor{gray}{gray}) and fed into the local features extraction module (\textcolor{orange}{orange}) which consists of convolutional neural networks and a multi-head self-attention block. The NetVLAD module (\textcolor{blue}{blue}) pools all the local features to generate the global descriptor.}
    \label{fig:nn_pipeline}
    \vskip -0.2in
\end{figure}

\section{Methods}
\label{sec:methods}

In this section, we show how submodular optimization can be used to reduce the computational burden of SLAM systems without sacrificing localization accuracy. 
In the first section, we use point cloud descriptors to select a diverse and memory efficient set of keyframes.
Next, we generate submaps by selecting keyframes which best constrain scan alignment using submodular optimization.
Finally, we describe our approach to quickly summarizing environments using streaming submodular optimization.

\subsection{Keyframe Selection via Feature Descriptor}
\label{subsec:keyframe_identification}

Selecting a keyframe from the stream of point clouds is a process that must balance localization performance—i.e., ensuring a candidate keyframe has sufficient overlap with saved keyframes for point cloud alignment—and memory usage, since any new keyframe must be kept in memory.
One could formulate a combinatorial optimization problem to select the optimal set of keyframes $\mathcal{K}^* \subseteq \textbf{E}$ given localization and memory constraints. 
While we present an offline solution to this optimization in \cref{sub:map_summarization}, the online version is more nuanced since keyframe selection occurs before all scans are acquired. 
Despite this causality constraint, submodularity can still be utilized to obtain a suboptimality bound for our keyframing strategy. 

Our approach relies on the neural network discussed in \cref{sec:LSNN} to efficiently determine the similarity between any two point clouds. 
We define the marginal value of any point cloud as the minimum Euclidean distance between the descriptor of a point cloud $E$ and any previously selected keyframe $\Delta f(\mathcal{K} \cup E) = \min_{K\in \mathcal{K}} || \phi(E) - \phi(K) ||$.
This marginal value inherently defines an objective function
\begin{equation}
\label{eq:keyframe_id_obj}
        f(\mathcal{K}) = \sum_{K_{i} \in \mathcal{K}} \min_{K_{j} \in \mathcal{K}_{<i}} ||\phi(K_{i}) - \phi(K_{j})||.
\end{equation}
Although we do not directly maximize \cref{eq:keyframe_id_obj}, we still obtain a suboptimality bound by only selecting keyframes with high marginal value because $f$ is a non-decreasing monotonic submodular set function, as proven in the following proposition. 

\begin{proposition}
    The set function $f: \mathcal{K} \rightarrow \mathbb{R}$ in \cref{eq:keyframe_id_obj} is a non-decreasing monotone submodular function.
\end{proposition}

\begin{proof}
    We begin by defining the value of the empty set $f(\emptyset) = 0$.
    Monotonicity is proven because the marginal value is the Euclidean distance to the nearest keyframe which much be positive.
    Submodularity is proven by verifying the marginal values using two keyframe sets $\mathcal{K}_{A} \subseteq \mathcal{K}_{B} \subseteq \textbf{E}$ follow \cref{def:submodularity} $\Delta f(\mathcal{K}_{A} \cup K) \geq \Delta f(\mathcal{K}_{B} \cup K) \implies \min_{K_{A} \in \mathcal{K}_{A}} ||\phi(K) - \phi(K_{A})|| \geq \min_{K_{B} \in \mathcal{K}_{B}} ||\phi(K) - \phi(K_{B})||$ which must be true since $\text{arg}\,\text{min}_{K_{A} \in \mathcal{K}_{A}} ||\phi(K) - \phi(K_{A})|| \in \mathcal{K}_{B}$, but $\text{arg}\,\text{min}_{K_{B} \in \mathcal{K}_{B}} ||\phi(K) - \phi(K_{B})|| \notin \mathcal{K}_{A}$.
    In other words, the minimizer w.r.t. $\mathcal{K}_{A}$ must be in $\mathcal{K}_{B}$ but the minimizer w.r.t. $\mathcal{K}_{B}$ may not belong to $\mathcal{K}_{A}$.
\end{proof}

We additionally require any keyframe set to be sufficient for localization and formalize this using degeneracy, a sensor- and environment-agnostic metric proposed in our previous work \cite{chen2023dliom}. 
To enforce this constraint, a keyframe is selected whenever $d = \frac{m^{2}}{\lambda_{min}(\mathcal{H}) \sqrt{z}} \geq \beta$, where $\lambda_{min}(\mathcal{H})$ is the minimum eigenvalue of the scan alignment algorithm (which we discuss in \cref{sub:submap_generation} thoroughly) and $m$ and $z$ are adaptive parameters defined in \cite{chen2023dliom}.

\begin{figure}
    \centering
    \vspace{6pt}
    \includegraphics[width=0.48\textwidth]{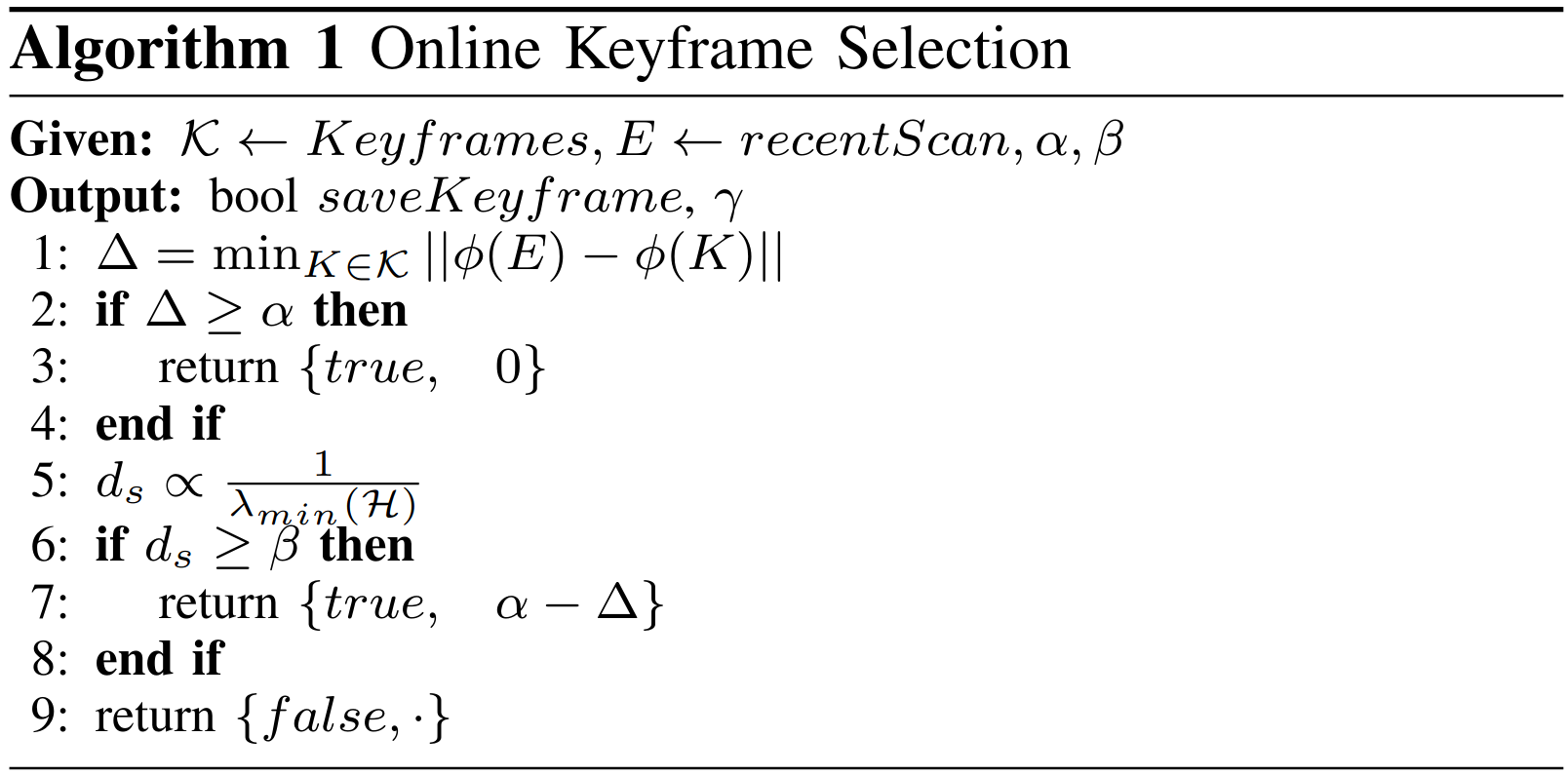}
    \label{alg:keyframe-selection}
    \vskip -0.45in
\end{figure}

Algorithm 1 describes our keyframe selection strategy, where a keyframe is selected according to a feature distance condition with threshold $\alpha$, and a degeneracy condition with threshold $\beta$.
In addition to a boolean variable representing whether to save a keyframe or not, we also return the added suboptimality (with respect to \cref{eq:keyframe_id_obj}) $\gamma$ such that the marginal value of the new keyframe is $\alpha - \gamma$.
In line 1 we find the smallest distance to any keyframe in descriptor space.
If the closest distance is greater than $\alpha$, we return true and report no additional suboptimality on line 3.
On lines 5-8, we compute degeneracy and save a keyframe if it exceeds the threshold $\beta$.
If neither condition is met, Algorithm 1 returns false.
By knowing the marginal value of each keyframe at the time of selection, we also maintain a known suboptimality bound which is given in the following proposition.

\begin{proposition}
    The keyframe set $\mathcal{K}$ generated by Algorithm 1 maintains a suboptimality bound relative to an optimal keyframe set $\mathcal{K}^{*}$ of $f(\mathcal{K}) \geq \frac{\alpha}{2} f(\mathcal{K}^{*}) - \sum_{K \in \mathcal{K}} \gamma_{K}$.
\end{proposition}

\begin{proof}
    The marginal value of each keyframe in an optimal set $\mathcal{K}^{*}$ is upper bounded by the maximum distance between any points on a unit hypersphere $\Delta f(\mathcal{K} \cup K) \leq 2 \rightarrow f(\mathcal{K}^{*}) \leq 2|\mathcal{K}|$.
    When keyframe $K$ is added, we know its marginal value is $\alpha - \gamma_{K}$.
    Thus, the total value is found by taking a sum over all placed keyframes $f(\mathcal{K}) = \sum_{K \in \mathcal{K}} (\alpha - \gamma_{K}) = \alpha |\mathcal{K}| - \sum_{K \in \mathcal{K}} \gamma_{K} \geq \frac{\alpha}{2} f(\mathcal{K}^{*}) - \sum_{K \in \mathcal{K}} \gamma_{K}$.
\end{proof}

\subsection{Submap Generation via Submodular Maximization}
\label{sub:submap_generation}

LiDAR odometry derives its accuracy from scan alignment algorithms which find the pose (translation $\textbf{t} \in \mathbb{R}^{3}$ and rotation $\textbf{q} \in \mathbb{S}^{3}$) that minimizes the nonconvex cost function 
\begin{equation}
\label{eq:alignment_obj}
    \sum_{c \in \mathcal{C}} \mathcal{E}(\textbf{t},\textbf{q},E,T),
\end{equation}
with the sum of error residuals $\mathcal{E}$ over all sufficiently close, i.e., corresponding, points in $\mathcal{C}$ between the scan $E$ and target scan $T$.
Computing an accurate solution to \cref{eq:alignment_obj} in real-time requires intelligently selecting the target scan, e.g., a submap $\mathcal{S} \subseteq \mathcal{K}$. 
We utilize the plane-to-plane cost function from \cite{segal2009generalized}.

\begin{figure}
    \centering
    \vspace{6pt}
    \includegraphics[width=0.48\textwidth]{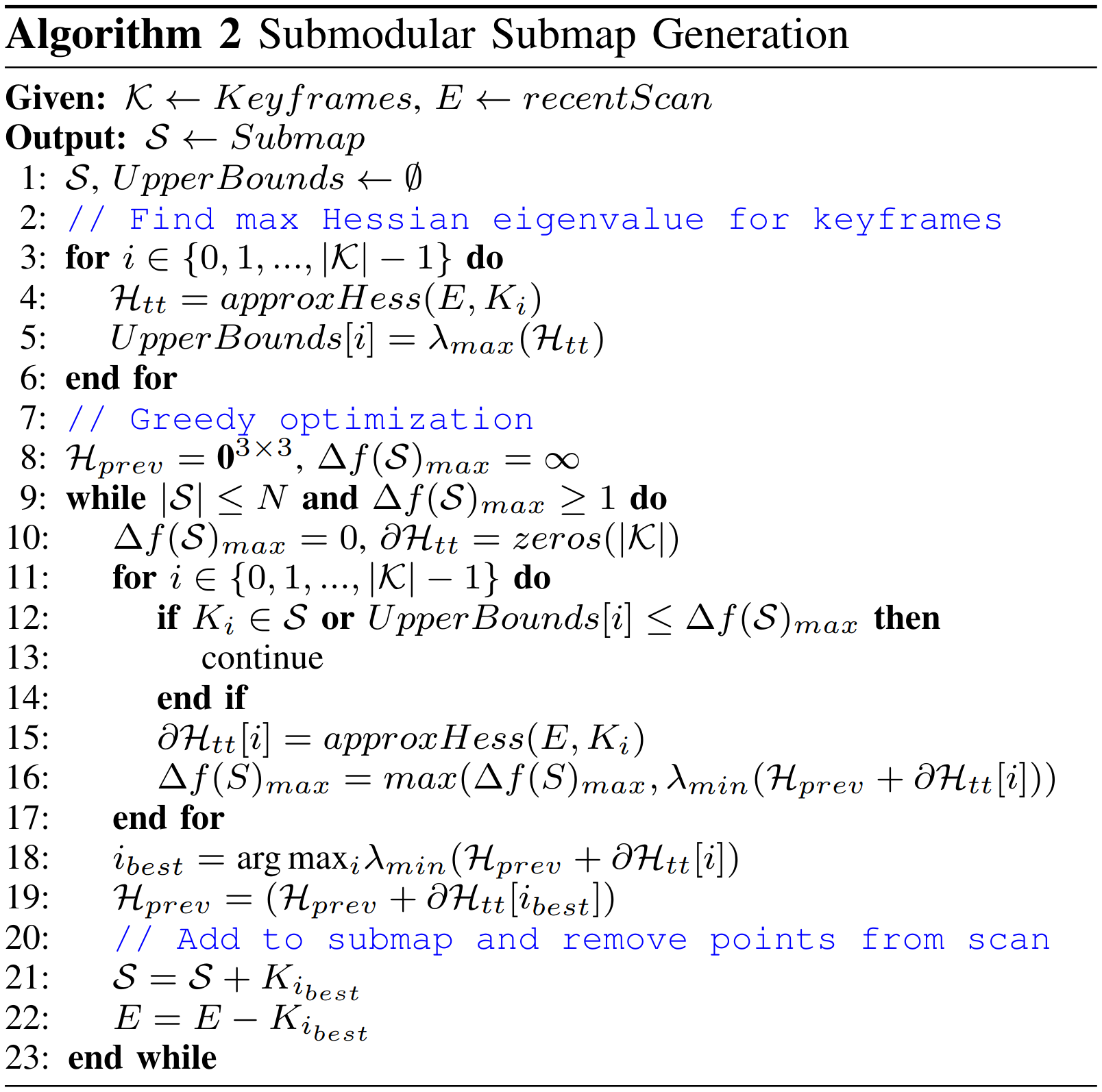}
    \label{alg:submap-generation}
    \vskip -0.45in
\end{figure}

Submap generation caries a trade-off between alignment accuracy and runtime, as previous work demonstrated the advantage of including corresponding points at the edge of the scan, but computation time scales with submap size \cite{chen2022direct}, \cite{xu2022fast}.
Poor scan alignment is primarily caused by self-similar environments, e.g., long featureless hallways, which introduce close local minima. 
The minimum eigenvalue associated with translation of the Hessian of \cref{eq:alignment_obj} $\lambda_{min}(\mathcal{H}_{tt})$ has been used to quantify how well constrained \cref{eq:alignment_obj} is \cite{ebadi2021dare, tagliabue2021lion, han2023dams}.
We assume the Hessian is available in analytic form by differentiating \cref{eq:alignment_obj} with respect to three translational variables to find a Jacobian $J_{t}$, which can be used to approximate the Hessian via $\mathcal{H}_{tt} \approx J_{t}^{T}J_{t}$.
Thus, our objective is to find the submap $\mathcal{S} \subseteq \mathcal{K}$ for a given scan $E$ that solves
\begin{equation}
\label{eq:submap_gen_obj}
    \begin{aligned}
        & \max_{\mathcal{S} \subseteq \mathcal{K}}f(\mathcal{S}) = \lambda_{min}(\mathcal{H}_{tt}(E, \mathcal{S})) \\
        & \text{Subject to}~|\mathcal{S}| \leq N,
    \end{aligned}
\end{equation}
which can be interpreted as finding the submap $\mathcal{S}$ that maximizes the minimum eigenvalue of $\mathcal{H}_{tt}$.
To make solving \cref{eq:submap_gen_obj} tractable, we make the following assumption.

\begin{assumption}
\label{assumption:hess_additive}
    The marginal value of adding a keyframe to the submap for \cref{eq:submap_gen_obj} is dominated by the additional value provided by new correspondences being made, i.e., points shared between the scan and new keyframe which did not exist in the previous submap.
\end{assumption}

Assumption~\ref{assumption:hess_additive} allows us to approximate $\mathcal{H}_{tt}$ as a sum of sub-Hessians.
We denote the sub-Hessian associated with keyframe $K_{i}$ as $\partial\mathcal{H}_{tt}[i]$ and approximate it as $\partial\mathcal{H}_{tt}[i] = \mathcal{H}_{tt}(E - \mathcal{S}_{<i}, K_{i})$, where $\mathcal{S}_{<i}$ is the submap including all keyframes with index less than $i$, and the subtraction of point clouds indicates removing the intersection of points ($A-B = A - (A\cap B))$.
Because each Hessian is the outer product of the Jacobian with itself, the Hessian must be positive semi-definite $\mathcal{H}_{tt} = J^{T}_t J_t \succeq 0$, which intuitively means each keyframe cannot make \cref{eq:alignment_obj} less constrained.
We define any keyframe which introduces enough new correspondences such that $\partial\mathcal{H}_{tt}[i] \succ 0$ as a \emph{constructive keyframe}.

\begin{proposition}
    The set function $f:\mathcal{S} \rightarrow \mathbb{R}$ in \cref{eq:submap_gen_obj} is an increasing monotone approximately submodular function under Assumption~\ref{assumption:hess_additive} as long as $N$ in the cardinality constraint $|\mathcal{S}| \leq N$ is small enough so each keyframe is constructive.
\end{proposition}

\begin{proof} 
    Our proof is similar to the proof of approximate submodularity in proposition 11 of \cite{carlone2018attention}.
    We prove \cref{eq:submap_gen_obj} is monotone increasing via the marginal gain
    \begin{align*}
        &\Delta f(\mathcal{S} \cup K_{i}) = \lambda_{min} (\mathcal{H}_{tt}(E,\mathcal{S} \cup K_{i})) - \lambda_{min} (\mathcal{H}_{tt}(E,\mathcal{S}))\\
        &= \lambda_{min} (\mathcal{H}_{tt}(E,\mathcal{S}) + \partial\mathcal{H}_{tt}[i]) - \lambda_{min} (\mathcal{H}_{tt}(E,\mathcal{S}))\\
        &= \min_{||\textbf{u}||=1}\textbf{u}^{T}[\mathcal{H}_{tt}(E,\mathcal{S}) + \partial\mathcal{H}_{tt}[i]]\textbf{u} - \min_{||\textbf{v}||=1}\textbf{v}^{T} \mathcal{H}_{tt}(E,\mathcal{S})\textbf{v} \\
        &\geq \Bar{\textbf{u}}^{T}[\mathcal{H}_{tt}(E,\mathcal{S}) + \partial\mathcal{H}_{tt}[i]]\Bar{\textbf{u}} - \Bar{\textbf{u}}^{T} \mathcal{H}_{tt}(E,\mathcal{S}) \Bar{\textbf{u}} \\
        &= \Bar{\textbf{u}}^{T}\partial\mathcal{H}_{tt}[i]\Bar{\textbf{u}} > 0,
    \end{align*}
    where $\Bar{\textbf{u}}$ is the minimizer of $\min_{||\textbf{u}||=1}\textbf{u}^{T}[\mathcal{H}_{tt}(E,\mathcal{S}) + \partial\mathcal{H}_{tt}[i]]\textbf{u}$, and the final line is true because the keyframe is constructive.
    The numerator of \cref{def:submod_ratio} is the minimum marginal gain attainable, so if the function is monotone increasing, the submodularity ratio must be bounded away from $0$ proving \cref{eq:submap_gen_obj} is approximately submodular.
\end{proof}

Algorithm 2 shows our approach to submap generation, which begins by finding and saving maximum Hessian eigenvalues (lines 3-6).
Note that we do not consider keyframes at least twice the sensor range from the current pose as they share no overlapping points.
Each greedy iteration (lines 9-23) computes the highest marginal value keyframe, then adds it to the submap and removes its points from the scan such that the following iterations will compute appropriate sub-Hessians.
The resulting submap enjoys a suboptimality bound from \cref{def:approx_submod_bound} if it reaches the cardinality constraint or achieves the maximum attainable value of \cref{eq:submap_gen_obj} if the marginal value approaches zero.
We only approximate Hessians using a percent ($\sim25\%$) of randomized scan points as previous work has shown this has a minimal effect on the relative magnitude of ICP Hessian eigenvalues \cite{gelfand2003geometrically}.

We save computation by eliminating the need to test most keyframes in each greedy iteration via an application of the classical Weyl inequalities which states
\begin{equation}
\label{eq:weyl}
    \lambda_{i+j-1}(A+B) \leq \lambda_{i}(A) + \lambda_{j}(B) \leq \lambda_{i+j-n}(A+B),
\end{equation}
where $A$ and $B$ are Hermitian matrices with ordered eigenvalues such that $\lambda_{1} \geq ... \geq \lambda_{n}$. 
We use Weyl's $1st$ inequality ($i=n,j=1$) to upper bound the minimum eigenvalue of the Hessian $\lambda_{min}(\mathcal{H}_{tt} + \partial\mathcal{H}_{tt}[i]) \leq \lambda_{min}(\mathcal{H}_{tt}) + \lambda_{max}(\partial\mathcal{H}_{tt}[i])$.
Next, we note that the maximum eigenvalues of the Hessians obtained on line 5 upper bound the sub-Hessian eigenvalues $\lambda_{max}(\mathcal{H}_{tt}(E, K)) \geq \lambda_{max}(\partial\mathcal{H}_{tt}[i])$. 
This is because the maximum eigenvalue is monotone non-decreasing (Weyl's $2nd$ inequality $i=1, j=n$) and the Hessian obtained on line 5 is the sum of the sub-Hessian with its counterpart formed by points shared by the scan, keyframe, and previous submap.
This combined with the minimum eigenvalue upper bound imply an upper bound on the marginal value of adding keyframe $K_{i}$ of $\Delta \lambda_{min}(\mathcal{S} \cup K) \leq  \, \lambda_{max}(\mathcal{H}_{tt}(E, K))$.
This means any keyframe with an upper bound smaller than the best marginal gain in the current iteration can be ignored.

\subsection{Map Summary via Streaming Submodular Optimization}
\label{sub:map_summarization}

Online keyframe selection must sufficiently constrain scan alignment, but too many keyframes yields unnecessarily dense maps that are not well suited for sharing across a team nor for downstream processes such as trajectory generation.
To address this issue, we propose a method to generate summary keyframe sets after completing a SLAM session using streaming submodular optimization.
This new feature can quickly generate a summary set of keyframes given a prescribed number of keyframes or memory limit, with computation time that only scales quadratically with the number of scans collected and notably not with the prescribed number of keyframes.
Our approach is also applicable to offline selection of keyframes from, e.g., a large dataset with recorded registered point clouds.

We formulate the LiDAR map summarization task as a k-mediod problem in which the goal is to minimize the sum of pairwise dissimilarities between keyframes and elements of the ground set (all registered scans).
We use the Euclidean distance between feature vectors and the k-mediod loss function \cite{kaufman2009finding} $L(\mathcal{K}) = \frac{1}{|\textbf{E}|}\sum_{E \in \textbf{E}} \min_{K \in \mathcal{K}} ||\phi(E) \, - \, \phi(K)||$, to measure how well a keyframe set summarizes an environment.
Minimizing this loss function is generally intractable for even small datasets, since it takes approximately $1\times 10^{-5}$ seconds to check the value of a given set, the brute force method would take over five years to find an optimal summary map of 10 keyframes for 100 total scans.
We assume an auxiliary element (scan) with a feature vector equal to the zero vector ($e_{0} = \textbf{0} \in \mathbb{R}^{256}$) and instead use a streaming submodular algorithm to maximize
\vskip -0.2in
\begin{align}
\label{eq:summary_obj}
    \max_{\mathcal{K} \in \textbf{E}}f(\mathcal{K})= L(e_0)-L(\mathcal{K} \cup e_0),
\end{align}
which is proven monotone submodular in \cite{kaufman2009finding}.

\begin{figure}
    \centering
    \vspace{6pt}
    \includegraphics[width=0.48\textwidth]{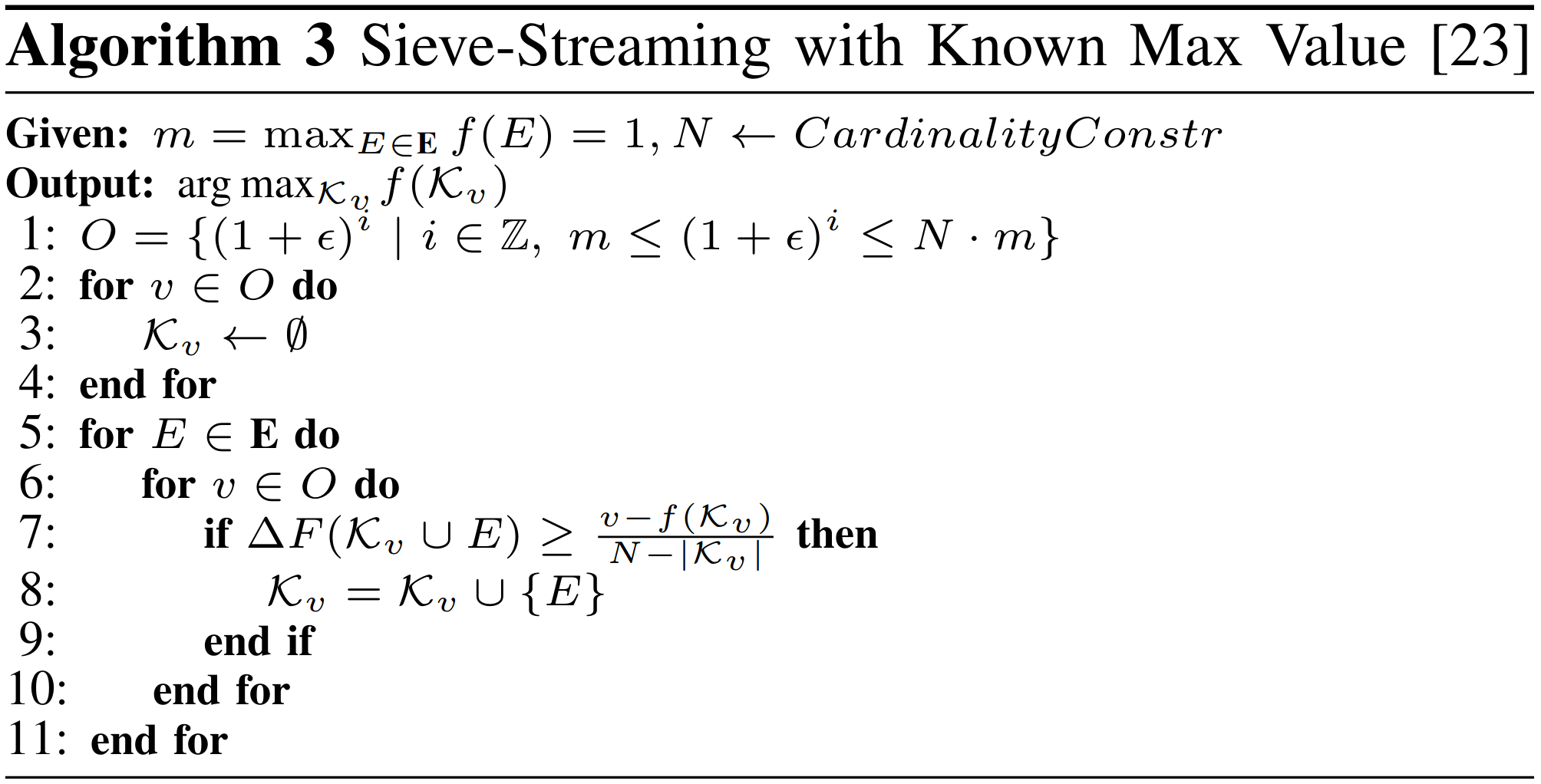}
    \label{alg:seive-streaming}
    \vskip -0.48in
\end{figure}

\begin{figure*}[t]
    \centering
    \vspace{6pt}
    \includegraphics[width=0.95\textwidth]{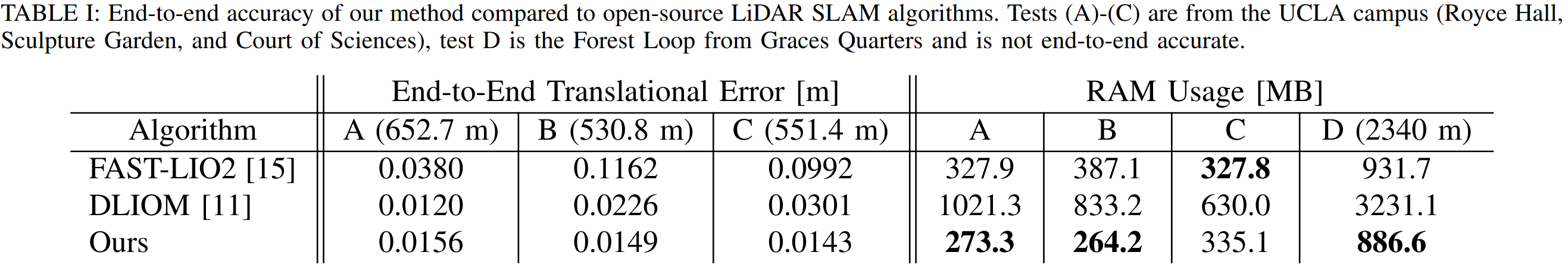}
    \label{tab:end-to-end_comp}
    \vskip -0.25in
\end{figure*}

We follow the seive-streaming algorithm described in \cite{badanidiyuru2014streaming} and given in Algorithm 3.
The maximum value of any single scan is $\max_{E \in \textbf{E}}f(e_{0} \cup E) = 1$ because each descriptor is mapped to a unit hypersphere and the auxiliary element is a distance of 1 from all scans.
We initialize each guess of the optimal value, $v$, and their corresponding solution sets, $O$ (lines 1-4).
The loop on lines 5-11 makes a single pass over all scans, adding a keyframe to a solution set if its marginal value exceeds a $v$-dependant threshold (line 7-8), with the function returning the highest value solution set.
The greedy maximization of \cref{eq:summary_obj} requires a pass over the dataset for each keyframe added (total of $N$) and the according complexity is $O(N \cdot |\textbf{E}|^{2})$.
On the other hand, the streaming submodular method only requires a single pass, improving the runtime complexity to $O(\frac{1}{\log \epsilon} |\textbf{E}|^2)$ and eliminating the need to have the entire dataset loaded in RAM at any point.

\section{Results}
\label{sec:results}

\begin{figure}
\vskip 0.05in
    \centering
    \includegraphics[width=0.48\textwidth]{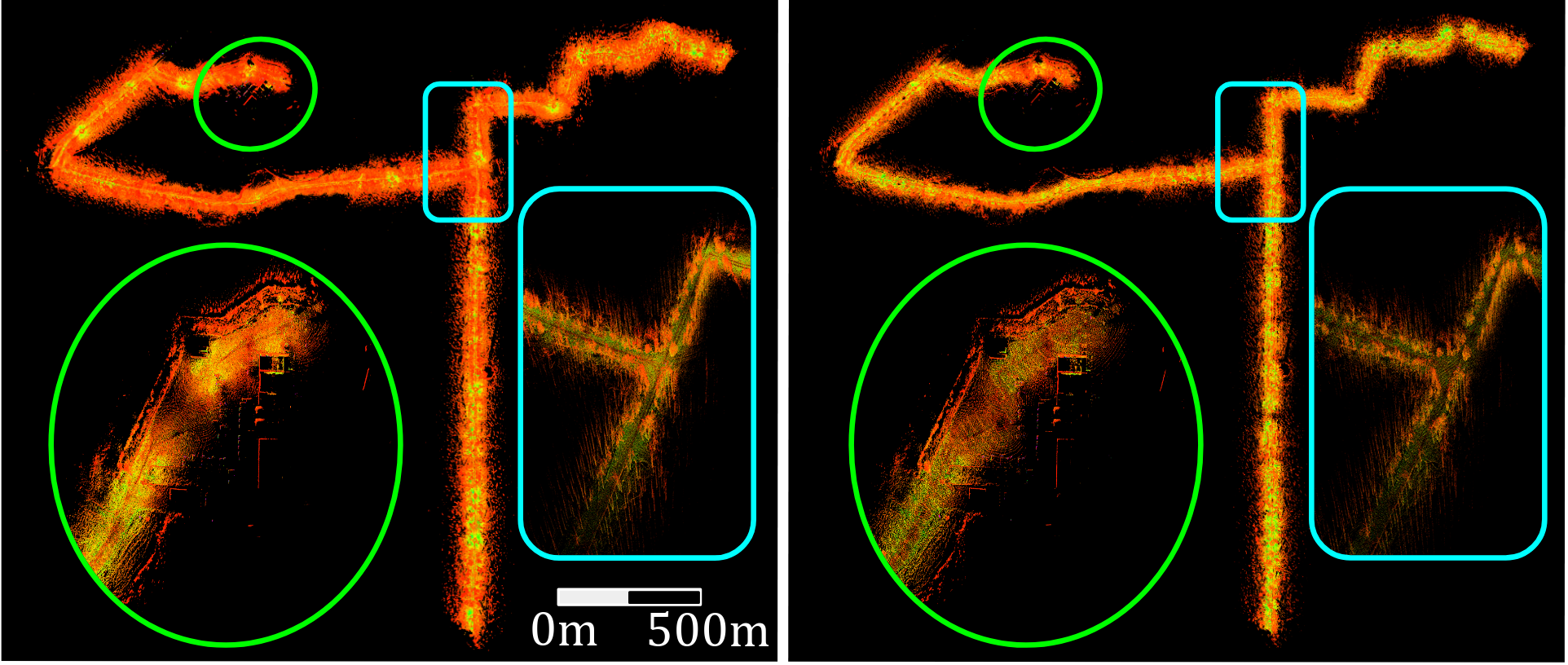}
    \caption{Map summaries for 6.2 km Mout Water Dataset from Graces Quarters. Left figure shows the full keyframe map using 527 keyframes and 186 MB. Right shows the summary map using 300 keyframes and 110 MB. Insets show detailed comparison of the maps.}
    \label{fig:summary_mout-water}
    \vskip -0.15in
\end{figure}

\begin{figure}
    \centering
    \includegraphics[width=.48\textwidth]{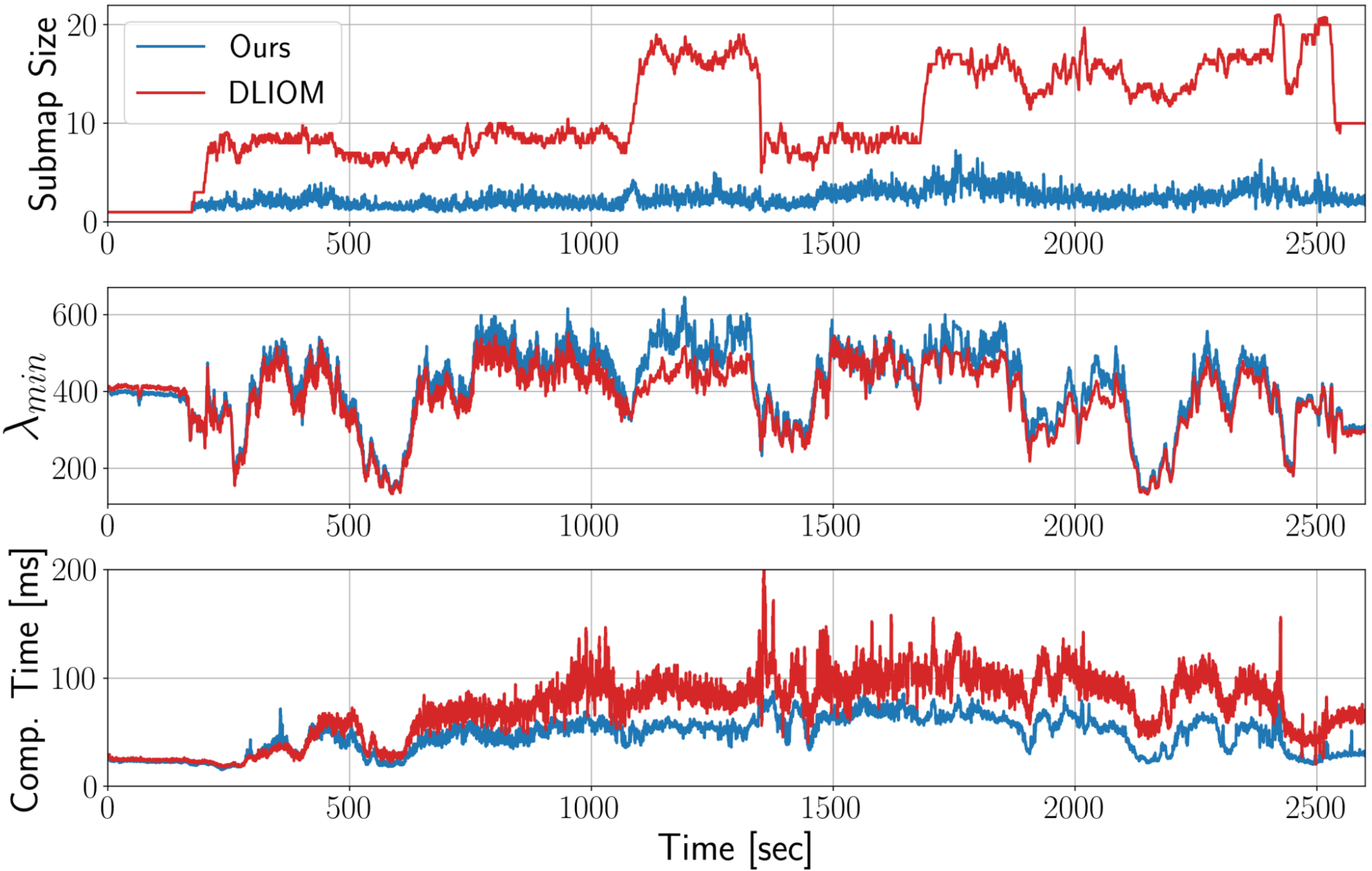}
    \caption{Submap generation comparison using 6.2 km Mout Water dataset using identical keyframe selection strategies. Top shows keyframes used per submap (DLIOM avg. 10.4 Ours avg. 2.3). Middle shows the minimum Hessian eigenvalue (DLIOM avg. 377 Ours avg. 407). Bottom shows per scan computation time (DLIOM avg. 67.3 ms Ours avg. 45.4 ms).}
    \label{fig:submap_evals}
    \vskip -0.24in
\end{figure}

We use a modified version of DLIOM \cite{chen2023dliom} where keyframes used for scan-to-submap alignment are the primary source of memory usage.
Our experiments use datasets collected at UCLA with a 10 Hz 32-beam Ouster OS0 and at the Army Research Laboratory Graces Quarters facility with a 20 Hz 128-beam Ouster OS1 mounted on a Clearpath Warthog.
Graces Quarters datasets include a 2.3 km loop and 6.2 km out-and-back shown in \cref{fig:top-right} and \cref{fig:summary_mout-water}.
The neural network was trained on UCLA campus data for 30 epochs.

\textit{Online Keyframe Selection.} The first experiment compares end-to-end accuracy and memory usage (measured by RSS) of our approach against DLIOM \cite{chen2023dliom} and a competitive open-source LIO algorithm \cite{xu2022fast}.
We note that memory allocation is a complex process which RSS does not entirely capture as running these algorithms on different hardware can affect RSS.
All results in Table 1 were made using an i7-10875H 8-Core CPU with 16 GB of RAM, with test D (Forest Loop) being included to demonstrate the memory demands of long-duration missions. 
We achieve comparable end-to-end accuracy with DLIOM despite using 64\% less memory on average.
This is due to a reduction in the number of keyframes selected, as we select 80\% less keyframes than DLIOM.
Our efficient approach to keyframe selection is shown in the bottom-right of \cref{fig:top-right} where eight keyframes each capture unique areas around a large sculpture.

\textit{Submodular Submap Generation.} Submap generation uses fewer keyframes to improve computation time without affecting scan alignment.
\cref{fig:submap_evals} shows a reduction from an average submap size of 10.4 keyframes using DLIOM to 2.3 using our method over a 6.2 km, 43 minute trajectory using an identical set of keyframes (placed every 5 m).
We achieve a slightly improved alignment constraint over the trajectory, improving from an average minimum Hessian eigenvalue of 377 to 407, while also decreasing the average computation time per scan from 67 to 45 ms.
This improvement is due to the submap generation selecting a smaller set of keyframes which can include distant keyframes with nonzero marginal value that heuristics might not consider.

\textit{Map Summarization.} \cref{fig:top-right} and \cref{fig:summary_mout-water} show the map summarization tool in use on datasets at Graces Quarters. 
The summary maps for the Forest Loop (2.3 km, 28.5 minutes) selected 75 and 300 keyframes and were generated in 0.67 and 0.69 seconds. 
The Mout Water (6.2 km, 43 minutes) summary map was generated using 300 keyframes in 1.65 seconds.
The insets in \cref{fig:summary_mout-water} show a similar level of detail despite the summary map using 40\% less memory.
The streaming approach improved the computation time of the summary map generation, as the same forest loop maps took 1.94 and 7.06 seconds respectively using a greedy approach.
\vskip -0.1in

\section{Conclusion}
\label{sec:conclusion}

We presented a novel online and offline solution for keyframe selection in addition to a new keyframe submap generation procedure.
Online keyframe selection identifies unique keyframes by generalizing the common distance threshold to feature space where each descriptor encodes point cloud similarity.
We maintain similar localization accuracy while reducing the average submap size and improving per scan computation time by optimizing submap generation for scan alignment constraint. 
Finally, our streaming submodular approach to map summarization generates size specified maps for efficient communication and downstream processes.

\noindent \textbf{Acknowledgements} \ \  We thank Jonathan Bunton and Matteo Marchi for introducing us to submodular optimization and Kenny Chen, Ryan Nemiroff, Stanley Wei, Nakul Joshi, Zihao Dong, and Jeffrey Pflueger for their support. We also thank Brian Kaukeinen from ARL for supporting testing.

\FloatBarrier
\bibliographystyle{IEEEtran}
\bibliography{references}

\begin{thebibliography}{10}
\providecommand{\url}[1]{#1}
\csname url@samestyle\endcsname
\providecommand{\newblock}{\relax}
\providecommand{\bibinfo}[2]{#2}
\providecommand{\BIBentrySTDinterwordspacing}{\spaceskip=0pt\relax}
\providecommand{\BIBentryALTinterwordstretchfactor}{4}
\providecommand{\BIBentryALTinterwordspacing}{\spaceskip=\fontdimen2\font plus
\BIBentryALTinterwordstretchfactor\fontdimen3\font minus \fontdimen4\font\relax}
\providecommand{\BIBforeignlanguage}[2]{{%
\expandafter\ifx\csname l@#1\endcsname\relax
\typeout{** WARNING: IEEEtran.bst: No hyphenation pattern has been}%
\typeout{** loaded for the language `#1'. Using the pattern for}%
\typeout{** the default language instead.}%
\else
\language=\csname l@#1\endcsname
\fi
#2}}
\providecommand{\BIBdecl}{\relax}
\BIBdecl

\bibitem{zhang2016submodular}
H.~Zhang and Y.~Vorobeychik, ``Submodular optimization with routing constraints,'' in \emph{Proceedings of the AAAI conference on artificial intelligence}, vol.~30, no.~1, 2016.

\bibitem{corah2019distributed}
M.~Corah and N.~Michael, ``Distributed matroid-constrained submodular maximization for multi-robot exploration: Theory and practice,'' \emph{Autonomous Robots}, vol.~43, pp. 485--501, 2019.

\bibitem{shi2021communication}
G.~Shi, I.~E. Rabban, L.~Zhou, and P.~Tokekar, ``Communication-aware multi-robot coordination with submodular maximization,'' in \emph{2021 IEEE International Conference on Robotics and Automation (ICRA)}.\hskip 1em plus 0.5em minus 0.4em\relax IEEE, 2021, pp. 8955--8961.

\bibitem{carlone2018attention}
L.~Carlone and S.~Karaman, ``Attention and anticipation in fast visual-inertial navigation,'' \emph{IEEE Transactions on Robotics}, vol.~35, no.~1, pp. 1--20, 2018.

\bibitem{khosoussi2019reliable}
K.~Khosoussi, M.~Giamou, G.~S. Sukhatme, S.~Huang, G.~Dissanayake, and J.~P. How, ``Reliable graphs for slam,'' \emph{The International Journal of Robotics Research}, vol.~38, no. 2-3, pp. 260--298, 2019.

\bibitem{doherty2022spectral}
K.~J. Doherty, D.~M. Rosen, and J.~J. Leonard, ``Spectral measurement sparsification for pose-graph slam,'' in \emph{2022 IEEE/RSJ International Conference on Intelligent Robots and Systems (IROS)}.\hskip 1em plus 0.5em minus 0.4em\relax IEEE, 2022, pp. 01--08.

\bibitem{ma2022overlaptransformer}
J.~Ma, J.~Zhang, J.~Xu, R.~Ai, W.~Gu, and X.~Chen, ``Overlaptransformer: An efficient and yaw-angle-invariant transformer network for lidar-based place recognition,'' \emph{IEEE Robotics and Automation Letters}, vol.~7, no.~3, pp. 6958--6965, 2022.

\bibitem{kretzschmar2011efficient}
H.~Kretzschmar, C.~Stachniss, and G.~Grisetti, ``Efficient information-theoretic graph pruning for graph-based slam with laser range finders,'' in \emph{2011 IEEE/RSJ International Conference on Intelligent Robots and Systems}.\hskip 1em plus 0.5em minus 0.4em\relax IEEE, 2011, pp. 865--871.

\bibitem{shan2020lio}
T.~Shan, B.~Englot, D.~Meyers, W.~Wang, C.~Ratti, and D.~Rus, ``Lio-sam: Tightly-coupled lidar inertial odometry via smoothing and mapping,'' in \emph{2020 IEEE/RSJ international conference on intelligent robots and systems (IROS)}.\hskip 1em plus 0.5em minus 0.4em\relax IEEE, 2020, pp. 5135--5142.

\bibitem{li2021towards}
K.~Li, M.~Li, and U.~D. Hanebeck, ``Towards high-performance solid-state-lidar-inertial odometry and mapping,'' \emph{IEEE Robotics and Automation Letters}, vol.~6, no.~3, pp. 5167--5174, 2021.

\bibitem{chen2023dliom}
K.~Chen, R.~Nemiroff, and B.~T. Lopez, ``Direct lidar-inertial odometry and mapping: Perceptive and connective slam,'' \emph{arXiv preprint arXiv:2305.01843}, 2023.

\bibitem{lin2023infola}
Y.~Lin, H.~Dong, W.~Ye, X.~Dong, and S.~Xu, ``Infola-slam: Efficient lidar-based lightweight simultaneous localization and mapping with information-based keyframe selection and landmarks assisted relocalization,'' \emph{Remote Sensing}, vol.~15, no.~18, p. 4627, 2023.

\bibitem{chen2022direct}
K.~Chen, B.~T. Lopez, A.-a. Agha-mohammadi, and A.~Mehta, ``Direct lidar odometry: Fast localization with dense point clouds,'' \emph{IEEE Robotics and Automation Letters}, vol.~7, no.~2, pp. 2000--2007, 2022.

\bibitem{chen2023direct}
K.~Chen, R.~Nemiroff, and B.~T. Lopez, ``Direct lidar-inertial odometry: Lightweight lio with continuous-time motion correction,'' in \emph{2023 IEEE international conference on robotics and automation (ICRA)}.\hskip 1em plus 0.5em minus 0.4em\relax IEEE, 2023, pp. 3983--3989.

\bibitem{xu2022fast}
W.~Xu, Y.~Cai, D.~He, J.~Lin, and F.~Zhang, ``Fast-lio2: Fast direct lidar-inertial odometry,'' \emph{IEEE Transactions on Robotics}, vol.~38, no.~4, pp. 2053--2073, 2022.

\bibitem{reinke2022locus}
A.~Reinke, M.~Palieri, B.~Morrell, Y.~Chang, K.~Ebadi, L.~Carlone, and A.-A. Agha-Mohammadi, ``Locus 2.0: Robust and computationally efficient lidar odometry for real-time 3d mapping,'' \emph{IEEE Robotics and Automation Letters}, vol.~7, no.~4, pp. 9043--9050, 2022.

\bibitem{zhang2016degeneracy}
J.~Zhang, M.~Kaess, and S.~Singh, ``On degeneracy of optimization-based state estimation problems,'' in \emph{2016 IEEE international conference on robotics and automation (ICRA)}.\hskip 1em plus 0.5em minus 0.4em\relax IEEE, 2016, pp. 809--816.

\bibitem{ebadi2021dare}
K.~Ebadi, M.~Palieri, S.~Wood, C.~Padgett, and A.-a. Agha-mohammadi, ``Dare-slam: Degeneracy-aware and resilient loop closing in perceptually-degraded environments,'' \emph{Journal of Intelligent \& Robotic Systems}, vol. 102, pp. 1--25, 2021.

\bibitem{tagliabue2021lion}
A.~Tagliabue, J.~Tordesillas, X.~Cai, A.~Santamaria-Navarro, J.~P. How, L.~Carlone, and A.-a. Agha-mohammadi, ``Lion: Lidar-inertial observability-aware navigator for vision-denied environments,'' in \emph{Experimental robotics: The 17th international symposium}.\hskip 1em plus 0.5em minus 0.4em\relax Springer, 2021, pp. 380--390.

\bibitem{han2023dams}
F.~Han, H.~Zheng, W.~Huang, R.~Xiong, Y.~Wang, and Y.~Jiao, ``Dams-lio: A degeneration-aware and modular sensor-fusion lidar-inertial odometry,'' in \emph{2023 IEEE International Conference on Robotics and Automation (ICRA)}.\hskip 1em plus 0.5em minus 0.4em\relax IEEE, 2023, pp. 2745--2751.

\bibitem{gelfand2003geometrically}
N.~Gelfand, L.~Ikemoto, S.~Rusinkiewicz, and M.~Levoy, ``Geometrically stable sampling for the icp algorithm,'' in \emph{Fourth International Conference on 3-D Digital Imaging and Modeling, 2003. 3DIM 2003. Proceedings.}\hskip 1em plus 0.5em minus 0.4em\relax IEEE, 2003, pp. 260--267.

\bibitem{chen2021anchor}
Y.~Chen, L.~Zhao, Y.~Zhang, S.~Huang, and G.~Dissanayake, ``Anchor selection for slam based on graph topology and submodular optimization,'' \emph{IEEE Transactions on Robotics}, vol.~38, no.~1, pp. 329--350, 2021.

\bibitem{badanidiyuru2014streaming}
A.~Badanidiyuru, B.~Mirzasoleiman, A.~Karbasi, and A.~Krause, ``Streaming submodular maximization: Massive data summarization on the fly,'' in \emph{Proceedings of the 20th ACM SIGKDD international conference on Knowledge discovery and data mining}, 2014, pp. 671--680.

\bibitem{feldman2018less}
M.~Feldman, A.~Karbasi, and E.~Kazemi, ``Do less, get more: Streaming submodular maximization with subsampling,'' \emph{Advances in Neural Information Processing Systems}, vol.~31, 2018.

\bibitem{nemhauser1978best}
G.~L. Nemhauser and L.~A. Wolsey, ``Best algorithms for approximating the maximum of a submodular set function,'' \emph{Mathematics of operations research}, vol.~3, no.~3, pp. 177--188, 1978.

\bibitem{das2011submodular}
A.~Das and D.~Kempe, ``Submodular meets spectral: greedy algorithms for subset selection, sparse approximation and dictionary selection,'' in \emph{Proceedings of the 28th International Conference on International Conference on Machine Learning}, 2011, pp. 1057--1064.

\bibitem{arce2023padloc}
J.~Arce, N.~V{\"o}disch, D.~Cattaneo, W.~Burgard, and A.~Valada, ``Padloc: Lidar-based deep loop closure detection and registration using panoptic attention,'' \emph{IEEE Robotics and Automation Letters}, vol.~8, no.~3, pp. 1319--1326, 2023.

\bibitem{komorowski2021minkloc3d}
J.~Komorowski, ``Minkloc3d: Point cloud based large-scale place recognition,'' in \emph{Proceedings of the IEEE/CVF Winter Conference on Applications of Computer Vision}, 2021, pp. 1790--1799.

\bibitem{vidanapathirana2022logg3d}
K.~Vidanapathirana, M.~Ramezani, P.~Moghadam, S.~Sridharan, and C.~Fookes, ``Logg3d-net: Locally guided global descriptor learning for 3d place recognition,'' in \emph{2022 International Conference on Robotics and Automation (ICRA)}.\hskip 1em plus 0.5em minus 0.4em\relax IEEE, 2022, pp. 2215--2221.

\bibitem{segal2009generalized}
A.~Segal, D.~Haehnel, and S.~Thrun, ``Generalized-icp.'' in \emph{Robotics: science and systems}, vol.~2, no.~4.\hskip 1em plus 0.5em minus 0.4em\relax Seattle, WA, 2009, p. 435.

\bibitem{kaufman2009finding}
L.~Kaufman and P.~J. Rousseeuw, \emph{Finding groups in data: an introduction to cluster analysis}.\hskip 1em plus 0.5em minus 0.4em\relax John Wiley \& Sons, 2009.

\end{thebibliography}

\end{document}